%% file: physweep_filled.tex
\pdfoutput=1 
\documentclass[11pt]{article}

\usepackage[font=libertinus, citestyle=numeric]{kurbanlab}

\input{affiliations}   

\usepackage{tikz}
\usetikzlibrary{arrows.meta,positioning,calc,backgrounds,fit,patterns}
\usepackage{pifont}
\usepackage{colortbl}   
\usepackage{multirow}

\definecolor{ForestGreen}{RGB}{34,139,34}
\definecolor{RoyalBlue}{RGB}{65,105,225}
\definecolor{BurntOrange}{RGB}{204,85,0}
\definecolor{Plum}{RGB}{142,69,133}

\definecolor{trkhi}{RGB}{31,102,160}
\newcommand{\tzero}{\cellcolor{black!8}\textcolor{black!55}{$0$}}
\newcommand{\tk}[2]{\cellcolor{trkhi!#1}$#2$}
\newcommand{\tkb}[2]{\cellcolor{trkhi!#1}$\mathbf{#2}$}

\newcommand{\PRE}{\ensuremath{\mathrm{PRE}}}
\newcommand{\PRI}{\ensuremath{\mathrm{PRI}}}
\newcommand{\CBI}{\ensuremath{\mathrm{CBI}}}
\newcommand{\EG}{\ensuremath{\mathrm{EG}}}
\newcommand{\thetah}{\ensuremath{\hat{\theta}}}

\title{\textsc{PhysWeep}: Does a Video Generator Realize the Physics You Ask For?}
\RunningTitle{PhysWeep: parameter faithfulness in video generators}

\Author{Rasul Khanbayov}{hbku}
\Author[corresponding=hkurban@hbku.edu.qa, orcid=0000-0003-3142-2866]{Hasan Kurban}{hbku}

\CodeURL{https://github.com/KurbanIntelligenceLab/physweep}

\begin{document}
\maketitle

\begin{abstract}
Image-to-video generators are often credited with absorbing physical
dynamics as implicit world models, a claim the community currently checks
with plausibility scores that ask whether a clip looks consistent with
real-world motion. Plausibility is the wrong test on its own, because a
clip can look natural while encoding the wrong value of the governing
physical parameter, and no existing benchmark measures this gap directly.
\textsc{PhysWeep} closes it with a fixed, label-free audit, treating a
frozen generator as a black box, recovering the realized parameter from
generated pixels, and reporting how often generation is trackable at all,
how far the realized value sits from the requested one, and which, if
either, of the literature's two proposed failure mechanisms the data
support. A deterministic-simulator positive control confirms every score
is exactly checkable.
Applied to three open generators across six sweep axes, \textsc{PhysWeep}
finds a specific, reproducible, previously undocumented failure.
Conditional on producing trackable motion, two of the three generate
confident, well-fit dynamics that converge to one of a small number of
fixed, wrong values selected by the \emph{sampling seed} rather than by
the request, reproducing across two independent model families, two
physical systems, and an independent tracker. It matches neither the
prior reversion nor the case-based clamping the literature anticipates,
because the reversion target is seed-conditional rather than a single
global default, and a leave-one-out selection rule rejects both; the
in-range faithfulness slope is statistically indistinguishable from zero
wherever a response is estimable at all. A benchmark averaging over seeds
would never see this: each sample is confidently locked to a wrong
constant, exactly the failure a plausibility score is structurally blind
to. We release the protocol, suite, and analysis code as a reusable,
model-agnostic audit for generators that claim physics controllability.
\end{abstract}

\section{Introduction}
\label{sec:intro}

Progress in video generation has revived an old hope: that a model trained
only to predict pixels will, as a by-product, learn the laws that govern how
the world moves \citep{brooks2024sora,agarwal2025cosmos}. The community has
accordingly invested in measuring the \emph{physical plausibility} of
generated video, through benchmarks that ask human raters or VLM auto-raters
whether motion looks right
\citep{bansal2024videophy,bansal2025videophy2,meng2024phygenbench,motamed2025physicsiq,zhang2025morpheus,phyworldbench2025}.

Plausibility is the wrong target on its own. A generator can produce a clip
that is visually flawless yet encodes an incorrect value of the relevant
parameter: a ball can fall along a smooth arc whose implied gravitational
acceleration is simply wrong, and no plausibility rater will object, because
the arc is locally consistent. The meaningful question is whether a generator
realizes the \emph{specific} physics implied by its conditioning, and whether
that realization extends past the values the model saw most often in
training. The second clause matters because a model that has internalized a
law should treat the controlling parameter as a free variable, whereas one
that has memorized typical motions falls back on what it has seen once pushed
away from the common case.

Recent roadmaps name this gap directly: surveys identify \emph{intrinsic
faithfulness} and \emph{controllability} as the frontier and note that
metrics for them remain scarce
\citep{roadmap2025visualworld,evolution2026videogen}, judge-based scores
are hard to audit because a judge's weights and APIs drift while a
parameter recovered by a fixed estimator does not
(\cref{sec:related}), and the gap is not incidental, since
generators optimized for media prioritize visual smoothness over physical
fidelity \citep{embodiedworld2026}. What is missing is a fixed, auditable
measurement of whether a \emph{conditioned} parameter is realized in the
pixels.

We turn ``physical understanding'' of a pretrained generator into a
falsifiable measurement. Treat the frozen model as a black-box map from a
conditioned parameter $\theta$ to a trajectory, recover the realized
$\thetah$ from generated pixels with no human labels, and quantify
faithfulness, out-of-range behavior, and the \emph{form} of any failure.
\textsc{PhysWeep} also adjudicates between the two mechanisms the literature
proposes: collapse to a single global default, or case-based mimicry of the
nearest seen value. Our delta against the closest prior art
(\cref{sec:related}) is a conjunction: a \emph{swept} control
parameter rather than a fixed default, on \emph{frozen black-box} models,
over \emph{six axes} rather than gravity alone, with a continuous
\emph{parameter-recovery} readout rather than a binary or judge score. We
do not claim to be first to recover physics from generated pixels. What is
new is the \emph{determinant}: the sampling seed is known to fix a
generator's appearance \citep{xu2025goodseed}, and we show it also fixes
the physics, overriding an explicit request for it.

\begin{kilkey}
The sampling seed, not the request, picks the physics. A benchmark that
averages over seeds reports a smooth mean and hides that each individual
sample is confidently locked to a wrong constant.
\end{kilkey}

\paragraph{Contributions}
\begin{enumerate}
  \item \textbf{A reusable audit artifact.} \textsc{PhysWeep} ships
  deterministic parameterized systems with ground-truth $\theta$, a fixed
  label-free estimator, and analysis code (\cref{sec:suite}),
  validated by a simulator-video positive control. The estimator is fixed
  rather than learned, so the measurement is auditable by construction
  (unlike evolving VLM-judge scores \citep{phyground2026}), and it applies
  to any generator accepting one of our two conditioning channels,
  including ones claiming physics controllability.
  \item \textbf{Label-free metrics and a mechanism adjudication.} A
  trackability rate, a Parameter Recovery Error (\PRE) with its
  Extrapolation Gap (\EG) and a monotonicity score, two competing
  failure-form indices, Prior-Reversion (\PRI) and Case-Based (\CBI), and a
  leave-one-out rule that lets the data choose between global-default
  collapse and nearest-value mimicry, or reject both
  (\cref{sec:metrics}) --- together turning ``did the generator
  learn physics'' into a falsifiable, pre-registered test rather than a
  qualitative impression.
  \item \textbf{A formal account of the failure we find, and the statistic
  that identifies it.} We define seed-conditional defaults and prove they
  force $\PRI\to1$ with a vanishing held-out $R^2$ (so rejecting both prior
  hypotheses is a derived prediction, not an unexplained negative), and that
  the same signature arises from noise, so \PRI\ alone cannot establish the
  mechanism, and give a Seed Determination Index that can
  (\cref{sec:seedtheory}). This closes a real gap in how the field
  reports results: an experiment that is not seed-resolved cannot tell
  seed-locking apart from a well-behaved generator with noisy output, and
  we show that distinction is exactly what is at stake here. We report
  what the index isolates on three frozen generators
  (\cref{sec:results}), extending the known
  imprint of the sampling seed on generated \emph{appearance}
  \citep{xu2025goodseed} to generated \emph{dynamics}.
\end{enumerate}

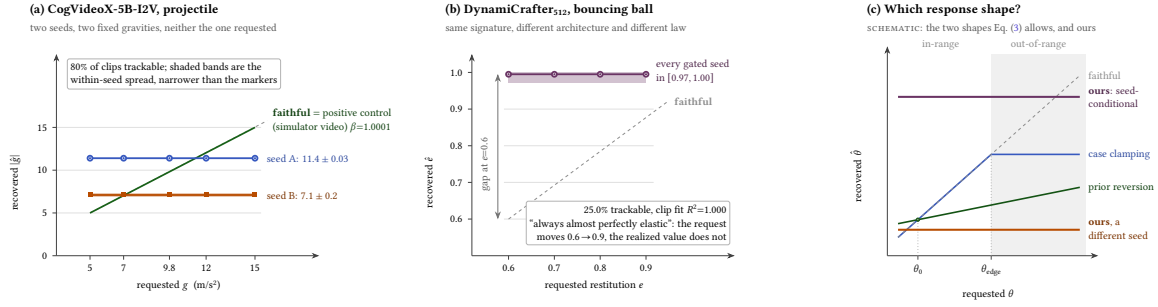
\begin{figure}[t]
\centering
\resizebox{\linewidth}{!}{%
\begin{tikzpicture}[font=\small,>={Stealth[length=2.4mm]},
  ax/.style={->,thick,black!75},
  gridln/.style={black!12,line width=0.4pt},
  ideal/.style={dashed,black!45,line width=0.9pt},
  ctrl/.style={ForestGreen!70!black,line width=1.6pt},
  seedA/.style={RoyalBlue!85!black,line width=1.7pt},
  seedB/.style={BurntOrange!90!black,line width=1.7pt},
  meas/.style={Plum!75!black,line width=1.7pt},
  note/.style={draw=black!30,fill=white,rounded corners=2pt,inner sep=3pt,
               font=\small,align=left},
  band/.style={opacity=0.30}]

\begin{scope}
  \node[anchor=west,font=\large\bfseries] at (-1.0,8.05) {(a) CogVideoX-5B-I2V, projectile};
  \node[anchor=west,font=\small,black!65] at (-1.0,7.45)
    {two seeds, two fixed gravities, neither the one requested};
  \foreach \y in {1.4,2.8,4.2}{\draw[gridln] (0,\y) -- (6.7,\y);}
  \draw[ax] (0,0) -- (8.6,0);
  \draw[ax] (0,0) -- (0,6.6);
  \foreach \g/\x in {5/1.1,7/2.2,9.8/3.7,12/4.9,15/6.5}
    {\draw[black!75] (\x,0) -- (\x,-0.13); \node[font=\small,anchor=north] at (\x,-0.16) {\g};}
  \foreach \v/\y in {0/0,5/1.4,10/2.8,15/4.2}
    {\draw[black!75] (0,\y) -- (-0.13,\y); \node[font=\small,anchor=east] at (-0.17,\y) {\v};}
  \node[font=\small,anchor=north] at (4.0,-0.66) {requested $g$ \ (m/s$^2$)};
  \node[font=\small,rotate=90,anchor=south] at (-1.05,2.4) {recovered $|\hat g|$};
  \draw[ideal] (1.1,1.4) -- (6.9,4.42);
  \draw[ctrl]  (1.1,1.4) -- (6.5,4.2);
  \fill[RoyalBlue!85!black,band] (1.1,3.1836) rectangle (6.5,3.2004);
  \draw[seedA] (1.1,3.192) -- (6.5,3.192);
  \foreach \x in {1.1,2.2,3.7,4.9,6.5}
    {\draw[seedA,fill=white] (\x,3.192) circle (2.3pt); \fill[RoyalBlue!85!black] (\x,3.192) circle (1.0pt);}
  \fill[BurntOrange!90!black,band] (1.1,1.932) rectangle (6.5,2.044);
  \draw[seedB] (1.1,1.988) -- (6.5,1.988);
  \foreach \x in {1.1,2.2,3.7,4.9,6.5}
    {\fill[BurntOrange!90!black] (\x-0.07,1.918) rectangle (\x+0.07,2.058);}
  \node[ForestGreen!45!black,font=\small,anchor=west,align=left] at (6.95,4.42)
    {\textbf{faithful} $=$ positive control\\ (simulator video) $\beta{=}1.0001$};
  \node[RoyalBlue!70!black,font=\small,anchor=west] at (6.75,3.192) {seed A: $11.4\pm0.03$};
  \node[BurntOrange!60!black,font=\small,anchor=west] at (6.75,1.988) {seed B: $7.1\pm0.2$};
  \node[note,anchor=north west] at (0.35,6.45)
    {$80\%$ of clips trackable; shaded bands are the\\
     within-seed spread, narrower than the markers};
\end{scope}

\begin{scope}[shift={(13.6,0)}]
  \node[anchor=west,font=\large\bfseries] at (-1.0,8.05) {(b) DynamiCrafter\textsubscript{512}, bouncing ball};
  \node[anchor=west,font=\small,black!65] at (-1.0,7.45)
    {same signature, different architecture and different law};
  \foreach \y in {1.2,2.4,3.6,4.8}{\draw[gridln] (1.05,\y) -- (5.95,\y);}
  \draw[ax] (0,0) -- (8.6,0);
  \draw[ax] (0,0) -- (0,6.6);
  \foreach \e/\x in {0.6/1.2,0.7/2.7,0.8/4.2,0.9/5.7}
    {\draw[black!75] (\x,0) -- (\x,-0.13); \node[font=\small,anchor=north] at (\x,-0.16) {\e};}
  \foreach \v/\y in {0.6/1.2,0.7/2.4,0.8/3.6,0.9/4.8,1.0/6.0}
    {\draw[black!75] (0,\y) -- (-0.13,\y); \node[font=\small,anchor=east] at (-0.17,\y) {\v};}
  \node[font=\small,anchor=north] at (4.0,-0.66) {requested restitution $e$};
  \node[font=\small,rotate=90,anchor=south] at (-1.15,2.4) {recovered $\hat e$};
  \draw[ideal] (1.2,1.2) -- (6.4,5.04);
  \node[black!45,font=\small,anchor=west] at (6.5,5.04) {\textbf{faithful}};
  \fill[Plum!75!black,band] (1.2,5.64) rectangle (5.7,6.00);
  \draw[meas] (1.2,5.94) -- (5.7,5.94);
  \foreach \x in {1.2,2.7,4.2,5.7}
    {\draw[meas,fill=white] (\x,5.94) circle (2.3pt); \fill[Plum!75!black] (\x,5.94) circle (1.0pt);}
  \node[Plum!55!black,font=\small,anchor=west,align=left] at (5.9,6.05)
    {every gated seed\\ in $[0.97,1.00]$};
  \draw[<->,black!55,line width=0.7pt] (0.85,1.2) -- (0.85,5.94);
  \node[black!65,font=\small,anchor=south,rotate=90] at (0.73,3.0) {gap at $e{=}0.6$};
  \node[note,anchor=south east,align=right] at (8.45,0.35)
    {$25.0\%$ trackable, clip fit $R^2{=}1.000$\\
     ``always almost perfectly elastic'': the request\\
     moves $0.6\!\to\!0.9$, the realized value does not};
\end{scope}

\begin{scope}[shift={(27.2,0)}]
  \node[anchor=west,font=\large\bfseries] at (-0.8,8.05) {(c) Which response shape?};
  \node[anchor=west,font=\small,black!65] at (-0.8,7.45)
    {\textsc{schematic}: the two shapes Eq.~(\ref{eq:pri}) allows, and ours};
  \fill[black!6] (3.4,0) rectangle (6.5,6.6);
  \draw[ax] (0,0) -- (6.9,0);
  \draw[ax] (0,0) -- (0,6.85);
  \node[black!45,font=\small,anchor=south] at (1.75,6.62) {in-range};
  \node[black!45,font=\small,anchor=south] at (4.95,6.62) {out-of-range};
  \node[font=\small,anchor=north] at (3.3,-0.98) {requested $\theta$};
  \node[font=\small,rotate=90,anchor=south] at (-0.72,2.6) {recovered $\hat\theta$};
  \draw[black!35,dotted,line width=0.7pt] (1.0,0) -- (1.0,1.179);
  \draw[black!35,dotted,line width=0.7pt] (3.4,0) -- (3.4,3.317);
  \draw[black!70] (1.0,0) -- (1.0,-0.13);
  \draw[black!70] (3.4,0) -- (3.4,-0.13);
  \node[font=\small,anchor=north] at (1.0,-0.17) {$\theta_0$};
  \node[font=\small,anchor=north] at (3.4,-0.17) {$\theta_{\text{edge}}$};
  \draw[RoyalBlue!85!black,line width=1.5pt] (0.35,0.6) -- (3.4,3.317) -- (6.3,3.317);
  \draw[ideal] (0.35,0.6) -- (6.3,5.90);
  \draw[ForestGreen!60!black,line width=1.5pt] (0.35,1.049) -- (6.3,2.239);
  \fill[ForestGreen!60!black] (1.0,1.179) circle (2.0pt);
  \draw[ForestGreen!60!black,fill=white] (1.0,1.179) circle (1.0pt);
  \draw[Plum!75!black,line width=1.9pt] (0.35,5.20) -- (6.3,5.20);
  \draw[BurntOrange!90!black,line width=1.9pt] (0.35,0.85) -- (6.3,0.85);
  \node[black!45,font=\small,anchor=west]            at (6.45,5.90) {faithful};
  \node[Plum!55!black,font=\small,anchor=west,align=left]        at (6.45,5.20) {\textbf{ours}: seed-\\conditional};
  \node[RoyalBlue!70!black,font=\small,anchor=west]   at (6.45,3.317) {case clamping};
  \node[ForestGreen!45!black,font=\small,anchor=west] at (6.45,2.239) {prior reversion};
  \node[BurntOrange!60!black,font=\small,anchor=west,align=left]  at (6.45,0.85) {\textbf{ours}, a\\different seed};
\end{scope}
\end{tikzpicture}}
\caption{\textbf{The seed picks the physics, and it is not tracking noise.}
\textbf{(a)} CogVideoX-5B-I2V, projectile, in-range grid, text conditioning:
recovered $|\hat g|$ is flat in the request and separated by \emph{sampling
seed}, with within-seed spread narrower than the plotted markers, while the
same tracker and fitter recover the swept parameter from simulator video at
slope $1.0001$. Each point is a mean over that seed's gated clips at that
$g$ ($80\%$ trackable in-range), not a single clip. \textbf{(b)} The same
signature in a different architecture on a different law: DynamiCrafter
returns near-perfect elasticity whatever restitution is asked for, at a fit
quality that rules out a tracking artifact ($25.0\%$ trackable in-range;
one fitter-degenerate outlier, $\hat e>1$, physically impossible for a
restitution coefficient, is excluded from the plotted line). \textbf{(c)}
Schematic, drawn to obey \cref{eq:pri}: case clamping tracks the
identity until $\theta_{\text{edge}}$ then goes flat, prior reversion is a
straight line of slope $\alpha$ crossing the identity at $\theta_0$, both
defined in \cref{sec:metrics} as functions of $\theta$ alone, so
neither can express a response flat in $\theta$ and offset by seed, which
is why the selection rule rejects both (\cref{prop:misspec}).
Sign convention in (a) is downward-negative; magnitudes are plotted.}
\label{fig:teaser}
\end{figure}

\section{Related work}
\label{sec:related}

A large family of benchmarks scores whether generated video looks physical:
human and learned raters of physical commonsense
\citep{bansal2024videophy,bansal2025videophy2}, broader law taxonomies
\citep{meng2024phygenbench,phyworldbench2025}, real footage
\citep{motamed2025physicsiq}, conservation-law metrics
\citep{zhang2025morpheus}, first-principles organization
\citep{t2vphysbench2025}, likelihood on valid-versus-invalid pairs
\citep{likephys2025}, and Physics and Controllability scored by VLM proxies
\citep{vbench2_2025,applepi2026}. These ask whether a clip looks right, not
whether it realizes a \emph{requested} parameter. \citet{phyground2026} argue
such judges are hard to reproduce because they evolve, and an evaluation on
real systems reports a fine-tuned judge scoring clean recordings only
moderately \citep{newtonianreal2026}; our plausibility contrast is an
independent instance on synthetic stimuli. Closer to our own readout,
\citet{li2025pisa} measure falling dynamics in generated video to drive
physics post-training, and \citet{thozhiyoor2026gravity} recover an effective
gravity from pixels, show that current generators under-accelerate toward a
fixed sub-Earth value, and eliminate metric-scale and frame-rate confounds
with a unit-free two-object timing test; \citet{gravityrewards2025} use a
similar readout as a verifiable post-training reward. This is the closest
prior art and we do not claim priority over it, but all three fix the target
law to gravity and characterize a model's default behavior, whereas our
question is orthogonal and complementary: given a parameter the user
\emph{asks for}, across six sweep axes, does the realized value move with the
request. Their fixed-$\theta$ result and our swept-$\theta$ result answer
different halves of the same worry.

A separate line establishes that the initial noise fixed by the sampling
seed leaves a strong, recoverable imprint on what a diffusion model
produces \citep{xu2025goodseed}, but for generated \emph{appearance};
\cref{def:scd} extends this to \emph{dynamics}, and to our
knowledge the seed has not previously been shown to fix a quantitative
physical parameter. A further wave of 2025--2026 methods claims physics
controllability through editable simulation, latent dynamics, geometry
guidance, force-vector or explicit-parameter conditioning, and
inference-time extrapolation
\citep{physchoreo2025,phantom2026,physvideo2026,phyco2026,motionforcing2026,phyparam2026,gillman2025forceprompting,lapg2026};
\textsc{PhysWeep} asks what a \emph{frozen} model does beforehand and
supplies the fixed measurement against which such claims can be
re-checked without training access. \Cref{app:related} extends this section
with the mechanism literature \citep{kang2025howfar,diffmemorize2025},
adjacent property-readout and VLM-based work
\citep{inferring2025dynamic,invisiblehand2026,vlipp2025,travl2025,vlmcannotreason2026},
and classic prediction/planning benchmarks
\citep{bear2021physion,bakhtin2019phyre,unterthiner2018fvd,huang2024vbench}.

\section{Problem: plausibility versus parameter faithfulness}
\label{sec:problem}

Let a system be governed by a known law $f$ with scalar control parameter
$\theta\in\Theta\subset\R$ (the framework extends to vector $\theta$;
we sweep scalars for clean identifiability). Given $s_0$ the true trajectory
is $\tau_\theta=f(s_0,\theta)$. A generator $G$ receives conditioning
$c(s_0,\theta)$ rendered from $\tau_\theta$ and produces $v=G(c(s_0,\theta))$;
we assume no access to its weights, data, or activations. An external
estimator $R$ maps generated pixels to $\thetah=R(v)$ by tracking the object
and fitting $f$.

\begin{definition}[Parameter faithfulness]
\label{def:faithful}
$G$ is \emph{$\epsilon$-faithful} on a sweep $S\subseteq\Theta$ if
$\E_{s_0}\,|R(G(c(s_0,\theta)))-\theta| \le \epsilon$ for all
$\theta\in S$.
\end{definition}

Plausibility, by contrast, is a property of $v$ alone: a rater $P(v)$
scores whether $v$ looks physical on some fixed scale (we use $1$--$5$),
independent of $\theta$. The two come apart precisely when $v$ is locally
smooth but encodes the wrong $\theta$.

\paragraph{Two hypotheses for out-of-range failure}
The literature offers two accounts of what a generator does when $\theta$
leaves the common range $S_\text{in}$, predicting different $\thetah$ that a
sweep can separate. Under \textbf{prior reversion},
$\thetah\approx\alpha\theta+(1-\alpha)\theta_0$ with a single fixed default
$\theta_0$ such as Earth gravity, independent of where in $S_\text{out}$ the
request lies. Under \textbf{case clamping} \citep{kang2025howfar}, $\thetah$
sticks at the nearest in-range value seen, $\theta_{\text{edge}}$. The two
coincide only when $\theta_0=\theta_{\text{edge}}$. Both assume $\thetah$ is
a function of $\theta$ alone; \cref{sec:seedtheory} shows our data
violate that assumption and derives the consequence.

\section{Metrics}
\label{sec:metrics}

Fix a system with sweep grid $\{\theta_i\}_{i=1}^{n}$ partitioned into an
in-range set $S_\text{in}$ (values common in natural video) and an
out-of-range set $S_\text{out}$ (values rare or absent). For each $\theta_i$
we draw $m$ seeded initial conditions and generations and obtain recovered
values $\{\thetah_{ij}\}_{j=1}^{m}$.

\paragraph{Parameter Recovery Error and Extrapolation Gap}
The normalized recovery error on a set $S$, with small $\delta>0$ for
stability, and the gap between splits, are
\begin{equation}
\label{eq:pre}
\PRE(S) = \frac{1}{|S|}\!\sum_{\theta_i\in S}\frac{1}{m}\sum_{j=1}^{m}
   \frac{|\thetah_{ij}-\theta_i|}{|\theta_i|+\delta},
\end{equation}
\begin{equation}
\label{eq:eg}
\EG = \PRE(S_\text{out})-\PRE(S_\text{in}).
\end{equation}
Lower \PRE\ is better and $\PRE=0$ is perfect faithfulness; $\EG\approx 0$
means the model extrapolates as well as it interpolates.

\paragraph{Trackability rate}
Not every generation contains motion a law can be fitted to. We report the
fraction of clips clearing the fit gate ($R^2\!\ge\!0.8$) as a metric in its
own right, not a silent exclusion: it is label-free and
\cref{sec:results} shows it carries information an automated
plausibility rater does not.

\paragraph{Faithfulness slope and monotonicity}
$\beta$ is the least-squares slope of $\thetah$ on $\theta$ over
$S_\text{in}$: a faithful model has $\beta\approx1$, one ignoring the
conditioning $\beta\approx0$. Because the response may be nonlinear we also
report Spearman's $\rho$, capturing whether the ordering of requested values
survives when the scale is off.

\paragraph{Two failure-form indices and a selection rule}
To adjudicate between the hypotheses of \cref{sec:problem} we fit both
to the out-of-range data:
\begin{align}
\text{prior reversion:}\quad & \thetah = \alpha\theta + (1-\alpha)\theta_0
   + \varepsilon, \label{eq:pri}\\
\text{case clamping:}\quad & \thetah = \min(\theta,\theta_{\text{edge}})
   + \varepsilon, \nonumber
\end{align}
\begin{equation*}
\PRI = 1-\hat\alpha,
\qquad
\CBI = 1 - \mathrm{RSS}_{\text{clamp}}/\mathrm{RSS}_{\text{null}}.
\end{equation*}
$\PRI=0$ means the request is followed, $\PRI=1$ means the output ignores it
and returns $\theta_0$, and \CBI\ is the out-of-range variance explained by
clamping at the edge. We select by leave-one-$\theta$-out $R^2$ and
$\Delta$AIC, reporting \emph{neither} when both fail the held-out test.
$\theta_0$ and $\theta_{\text{edge}}$ are reported, not assumed.

\paragraph{Stochasticity, scale, and uncertainty}
Generators are sampled with $m\ge5$ seeds per condition. We separate the bias
$\E[\thetah]-\theta$ from the within-condition dispersion
$\mathrm{std}(\thetah)$, so an unbiased-but-noisy model is not confused with
a biased-but-consistent one. Because we render the conditioning, spatial and
temporal scales are fixed within a clip and all metrics use the ratio
$\thetah/\theta$, invariant to a common rescaling; a scale-invariance
ablation checks this. Metric-scale ambiguity is a known confound for physics
readouts from generated video \citep{thozhiyoor2026gravity} and this is our
defense against it. All metrics carry $95\%$ bootstrap CIs over
$(\theta_i,\text{seed})$ pairs.

\paragraph{Well-definedness of \texorpdfstring{$\thetah$}{theta-hat}}
The metrics mean something only if $R$ recovers a well-defined quantity;
\cref{app:proof} gives the two regularity conditions (trackability,
identifiability) under which $\thetah$ converges in probability to the
realized parameter, a claim about the \emph{estimator}, not the
generator, and discloses one small bias this exposes for the damped
pendulum. When the motion is not well described by $f$ the residual is
large, so we gate at $R^2\!\ge\!0.8$ and treat low-fit generations as
\emph{non-physical} rather than forcing a $\thetah$.

\subsection{Seed-conditional defaults}
\label{sec:seedtheory}

Both hypotheses assume the realized value is a function of $\theta$ alone.
Our data are not, so we state the general model and derive what it implies
for the indices rather than treating the resulting ``neither'' as an
unexplained negative. This is our only novel formal claim.

\begin{definition}[Seed-conditional default]
\label{def:scd}
Let the sweep be run as a \emph{crossed} design, every grid value with every
seed, so that the seed is independent of $\theta$. $G$ has a
\emph{seed-conditional default} on $S$ if there is a non-constant
$c:\mathcal{Z}\to\Theta$ with $\thetah(\theta,z)=c(z)+\varepsilon$, where
$\varepsilon$ has mean $0$ and variance $\sigma_\varepsilon^2$ and is
independent of $(\theta,z)$, and $\sigma_c^2=\operatorname{Var}_z[c(z)]>0$.
Prior reversion with a single global $\theta_0$ is the special case
$\sigma_c^2=0$, $c\equiv\theta_0$, so the model nests it. The map $c$ is the
dynamical analogue of the appearance imprint the seed is known to carry
\citep{xu2025goodseed}.
\end{definition}

\begin{proposition}[The indices are misspecified, predictably]
\label{prop:misspec}
Under \cref{def:scd}: (i) the pooled least-squares prior-reversion
fit has $\hat\alpha\to0$ in probability, so $\PRI\to1$; (ii) its implied
default converges in probability to $\E_z[c(z)]$, a population mean
that no individual generation realizes; and (iii) every predictor that is a
function of $\theta$ alone has population $R^2=0$, so its
leave-one-$\theta$-out $R^2$ converges in probability to $0$ and, for any
fixed margin $\tau>0$, a selection rule accepting only
$R^2_{\mathrm{LOO}}>\tau$ returns \emph{neither} with probability tending to
one. Statements (i)--(iii) hold verbatim when $\sigma_c^2=0$ and
$\sigma_\varepsilon^2$ is large, so $\PRI$ and the selection outcome cannot
distinguish a seed-conditional default from unstructured noise.
\end{proposition}

\begin{proofsketch}
The crossed design gives $\thetah\perp\theta$, hence
$\operatorname{Cov}(\theta,\thetah)=0$, which yields (i) and, through the
intercept, (ii). For (iii) the best $\theta$-only predictor is the marginal
mean, so the population $R^2$ is $0$ and the held-out estimate converges to
it. The final sentence follows because none of the three arguments uses
$\sigma_c^2>0$. Full proof in \cref{app:seed}.
\end{proofsketch}

The margin in (iii) is not a detail: with $R^2_{\mathrm{LOO}}$ centered on
zero its \emph{sign} is noise, so a rule accepting any positive value would
admit unstructured data as prior reversion (our rule uses $\tau=0.1$;
justification and a smoke-test check in \cref{app:seed}). The proposition's
last
sentence is why a further statistic is needed: separating the regimes
requires one that conditions on the seed.

\paragraph{Seed Determination Index}
Decompose $\operatorname{Var}(\thetah)$ over the design's two factors:
$\mathrm{FDI}$ is the share from the requested $\theta$, $\mathrm{SDI}$ the
share from the seed, the latter a one-way random-effects intraclass
correlation so it is not upward-biased at this design size. Under
\cref{def:scd}, $\mathrm{FDI}\to0$ and
$\mathrm{SDI}\to\sigma_c^2/(\sigma_c^2+\sigma_\varepsilon^2)$ in probability;
when instead $\thetah=\theta+\varepsilon$ with $\varepsilon$ independent of
the seed, the two swap. On synthetic ground truth the released implementation
returns $\mathrm{SDI}=0.999$ for a seed-locked generator against $0.000$
faithful, $0.000$ pure noise and $0.060$ single global default, and refuses
rather than guessing below two seeds.

\section{Suite, pipeline, and protocol}
\label{sec:suite}

Each system is a deterministic renderer producing conditioning frames,
ground-truth $\theta$, and one interpretable sweep axis. Scenes are simple
and high-contrast so the tracker is robust and the identifiability conditions
above hold easily, keeping the measurement rather than the perception in
focus.
That has a cost \cref{sec:limitations} takes seriously: minimalist
synthetic scenes sit far from every tested model's training distribution.
Five systems supply six axes, the pendulum contributing two.

\begin{table}[t]
\centering
\caption{\textbf{The \textsc{PhysWeep} suite.} Five deterministic
renderers give six sweep axes, each with a closed-form map from a tracked
statistic to the parameter, so recovery is exact and label-free. Units: $g$
m/s$^2$, $\omega$ rad/s, $k$ N/m; $e,\zeta,\mu$ dimensionless. Rendering
settings and per-pair seed counts are in \cref{app:systems}.}
\label{tab:suite}
\small
\setlength{\tabcolsep}{4pt}
\resizebox{\linewidth}{!}{%
\begin{tabular}{@{}llll@{}}
\toprule
\kilth{System} & \kilth{$\theta$} & \kilth{Recovered statistic} & \kilth{Sweep grid: in-range \,/\, out-of-range} \\
\midrule
Projectile      & $g$      & parabola curvature       & $\{5,7,9.8,12,15\}$ \,/\, $\{1.6,3,20,25\}$ \\
Damped pendulum & $\omega$ & zero-crossing period     & $\{2,3,4,5\}$ \,/\, $\{0.5,1,8,12\}$ \\
                & $\zeta$  & envelope decay           & $\{.02,.05,.1\}$ \,/\, $\{.3,.6,.9\}$ \\
Bouncing ball   & $e$      & bounce-height ratio      & $\{.6,.7,.8,.9\}$ \,/\, $\{.2,.35,.97\}$ \\
Spring--mass    & $k$      & oscillation period       & $\{20,40,60,80\}$ \,/\, $\{5,10,200,400\}$ \\
Inclined slide  & $\mu$    & along-slope acceleration & $\{.1,.2,.3,.4\}$ \,/\, $\{.01,.7,1.0\}$ \\
\bottomrule
\end{tabular}}
\end{table}

\begin{figure}[t]
\centering
\kilgraphics{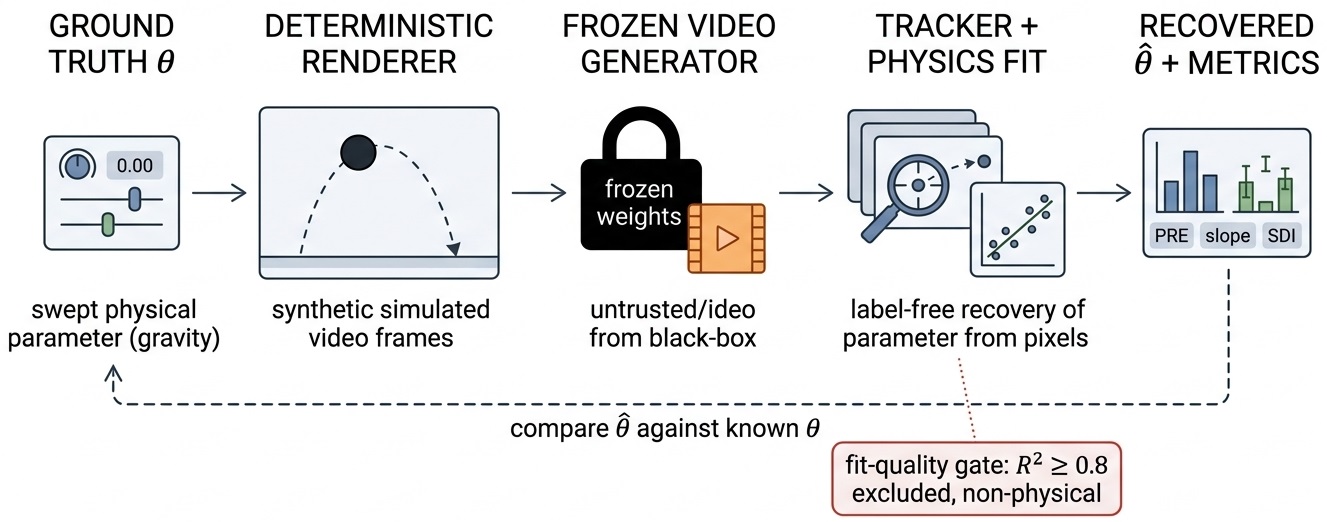}
\caption{\textbf{The measurement pipeline.} A known $\theta$ drives a
deterministic renderer, whose conditioning frames are handed to the
frozen generator under audit; the resulting clip is tracked and fit in
closed form to recover $\thetah$, which is compared back against the
known $\theta$ to compute every metric in \cref{sec:metrics}. The
fit-quality gate ($R^2\!\ge\!0.8$) is the only point where a generation
can be excluded, and it is excluded as \emph{non-physical}, not silently
dropped. Only the generator (dark box) is untrusted; every other stage is
exact and label-free by construction.}
\label{fig:pipeline}
\end{figure}

\paragraph{Pipeline}
\Cref{fig:pipeline} sketches the loop described here.
For each grid value $\theta_i$ and seed $z_j$ we render conditioning frames,
generate a continuation, track the observable and fit $f$ in closed form to
obtain $(\thetah_{ij},R^2_{ij})$; clips with $R^2_{ij}<0.8$ are marked
non-physical and excluded from \PRE, and that rate is reported alongside
every result. A VLM rater scores the same clips. Rendering uses a small
deterministic 2D engine with no external assets; tracking is color-blob
centroiding, with CoTracker3 \citep{karaev2023cotracker} as an independent
second tracker in ablation. Pseudocode is in \cref{app:pipeline}. The pipeline is
inference-only, run on a shared NVIDIA A100-SXM4 (40\,GB), not a single
consumer GPU (CogVideoX's $25.3$\,GB peak, with tiling, exceeds a 24\,GB
budget), for an estimated $\sim\!57$ total GPU-hours across all runs; the
estimate's derivation and per-model figures are in \cref{app:systems}.

\paragraph{How \texorpdfstring{$\theta$}{theta} is conditioned}
Models accept different inputs, so we treat the channel as a controlled axis.
\textbf{Frame-implied conditioning} supplies rendered frames whose motion
already implies $\theta$ and asks the model to continue, isolating dynamics
from semantics. \textbf{Text-specified conditioning} supplies a first frame
plus a prompt naming the regime, for example on the Moon versus on Jupiter,
testing whether a named control moves the realized parameter.
\Cref{sec:results} reports which channel each result uses, since only
one model admits both.

\paragraph{Protocol}
We tested three frozen open generators, each within reach of a single
high-end GPU (VRAM figures in \cref{app:systems}): LTX-Video
\citep{hacohen2024ltxvideo}, CogVideoX-5B-I2V \citep{yang2024cogvideox},
DynamiCrafter\textsubscript{512} \citep{xing2024dynamicrafter}. Stable Video
Diffusion \citep{blattmann2023svd} was dropped before any run because its
pipeline accepts no text prompt; a quantized Wan \citep{wan2025} was
considered and not run. Plausibility is contrasted against, not used as, the
main metric via one VLM auto-rater \citep{qwen25vl}; \emph{no human subset
and no \textsc{VBench-2.0} score were collected}. Before running we
pre-committed which of four outcomes would confirm or refute each
hypothesis, whatever the sweep actually showed: (1) faithful in-range,
fails out-of-range, prior reversion or case clamping selected --- the
field's own anticipated headline; (2) no conditioning effect
($\beta\approx0$); (3) faithful even out-of-range ($\EG\approx0$), a
surprising positive result reported as such; or (4) neither hypothesis
fits. Whichever occurred would drive the abstract, title, and
contributions; \cref{sec:results} names which two of the four the
sweep produced. We use $m\ge5$ seeds per $(\theta_i,\text{model})$, report
$95\%$ bootstrap CIs and effect
sizes rather than $p$-values alone, test
$\PRE(S_\text{out})>\PRE(S_\text{in})$ by
paired bootstrap, and report held-out $R^2$ and $\Delta$AIC for selection.
Model repositories and licenses are in \cref{app:systems}, which also
discloses that
no explicit commit/revision was pinned when loading any model, so each
resolved to that repository's default-branch head at run time;
\cref{app:checklist} is the reproducibility checklist.

\section{Results}
\label{sec:results}

Two protocol deviations are stated once and carried through.
\emph{First}, every result except the conditioning-channel ablation uses
\textbf{text-specified conditioning}: LTXConditionPipeline, the
\texttt{diffusers} component giving first-and-last-frame conditioning for
LTX-Video, never computes the scheduler's required shift parameter in the
release used here and, once patched, still produces incoherent output
regardless of resolution or scene content, confirmed against the working
single-image pipeline on an identical input; CogVideoX's official pipeline
accepts only one conditioning image. \emph{Second},
DynamiCrafter\textsubscript{512} replaced Stable Video Diffusion as above.

\begin{table}[t]
\centering
\caption{\textbf{Coverage on the three axes short clips can
identify at all}: fraction of in-range clips whose tracked trajectory
clears the $R^2\!\ge\!0.8$ fit gate, under text conditioning. Shading is
proportional to trackability, gray marks a measured zero (CogVideoX,
inclined slide), and \textbf{bold} marks the five cells yielding enough
gated clips for a slope. The remaining three axes are frame-budget-bound
rather than a conditioning result (\cref{sec:limitations}) and are
omitted here.}
\label{tab:coverage}
\small
\setlength{\tabcolsep}{5pt}
\renewcommand{\arraystretch}{1.22}
\begin{tabular}{@{}l ccc@{}}
\toprule
\kilth{Trackability (\%)} & \kilth{proj.} & \kilth{slide} & \kilth{ball} \\
\midrule
LTX-Video      & \tkb{20}{20}   & \tkb{35}{35}   & \tk{5}{5}    \\
CogVideoX      & \tkb{80}{80}   & \tzero         & \tk{5}{5}    \\
DynamiCrafter  & \tkb{64}{64}   & \tkb{80}{80}   & \tk{25}{25}  \\
\bottomrule
\end{tabular}
\end{table}

\paragraph{The instrument works: simulator-video positive control}
Ground-truth simulator continuations, not model output, through the identical
tracker and fitter, pooled across all six axes and both splits ($203$
measurements), give $\PRE(S_\text{in}) = 0.0031$ (CI $[0.0019, 0.0045]$),
slope $\beta = 1.0001$ (CI $[1.00008, 1.00023]$), and mean fit $R^2 = 0.986$.
The estimator recovers a swept parameter when one is present, so a null slope
on model output is a property of the model or the stimulus, not the pipeline.

\paragraph{Conditioning does not move the realized parameter}
\Cref{tab:coverage} shows the three axes short clips can identify at
all; trackability never exceeds $80\%$ there. The nine cells on the other
three, frame-budget-bound axes (not shown in the table) are a separate,
architectural $0\%$ (\cref{sec:limitations}), so ten of all
eighteen model-axis pairs measure $0\%$ trackability in total, for two
different reasons. Only five cells support a slope estimate, and every one is
indistinguishable
from zero while excluding $\beta=1$: LTX-Video projectile
$-0.015$ [$-0.090$, $0.039$] and inclined slide $-0.464$ [$-1.751$,
$1.097$]$^{\dagger}$; CogVideoX projectile $0.070$ [$-0.168$, $0.305$];
DynamiCrafter projectile $-0.025$ [$-0.466$, $0.480$] and inclined slide
$0.521$ [$-6.185$, $6.850$]$^{\dagger}$. The two marked $\dagger$ also
contain $\beta=1$, showing lack of power rather than absence, so we do not
count them as evidence for the null; the remaining three are inconsistent
with faithful conditioning. Recovery error is correspondingly large on
the projectile cells ($\PRE(S_\text{in})=0.853$, $2.016$, $0.781$ for
LTX-Video, CogVideoX, DynamiCrafter, against $0.0031$ for the positive
control), and rank order fares no better: Spearman's $\rho$ over the full
projectile sweep is $-0.268$ (LTX-Video) and $0.046$ (CogVideoX), so not
even the \emph{ordering} of requested values survives.

\paragraph{The failure is seed-locked, not untrackable noise}
The three models reach the same null slope through visibly different
behavior, confirmed by frame inspection. \textbf{LTX-Video} mostly fails to
generate any real fall: the conditioned disk stays near its initial position
regardless of the requested $g$ or how the prompt describes it, tested up to
``extremely strong gravity, like Jupiter, violent rapid acceleration
downward.'' \textbf{CogVideoX and DynamiCrafter} instead generate confident,
well-tracked motion that is simply the wrong amount, with recovered $\hat g$
clustering tightly by sampling seed and close to independent of the request:
on CogVideoX one seed gives $\hat g \approx -11.4 \pm 0.03$ across every
in-range $g\in\{5,7,9.8,12,15\}$ and another gives $-7.1\pm0.2$
(\cref{fig:teaser}), reproducing out-of-range on the same seeds.

That second mode also rules out the reading that the trackability rate
invites, that these models simply cannot render minimalist synthetic
scenes and the survivors are flukes: if the stimulus were merely
un-renderable, the recovered values clearing the gate would be scattered,
not tight. They are not (within-seed dispersion $0.03$--$0.13$,
individual-clip fit $R^2$ up to $1.000$, clusters reappearing under an
independent tracker). Confident, well-fit, seed-locked motion is what a
model ignoring $\theta$ produces, not what one failing to render produces,
so we state the headline conditionally: \textbf{conditional on producing
physically trackable motion, none of the three models detectably honors a
text-specified physical parameter.} Trackability bounds the
\emph{coverage} of that claim, not its validity.

\paragraph{The signature crosses architectures and physical systems}
The pattern is not specific to the projectile. LTX-Video's near-total-inertia
failure generalizes, with four of its five remaining axes below $6\%$
trackable. CogVideoX's seed-determined convergence generalizes only where
clips survive the gate, which fails entirely on spring-mass and inclined
slide; on bouncing-ball, every gated seed converges to a recovered
restitution in $[0.97,1.00]$ regardless of the request, and DynamiCrafter
independently reproduces the identical signature on the same system
($\hat e = 0.99$--$1.00$ across the full grid, individual-clip
$R^2=1.000$): two architectures landing on the same wrong-but-confident
behavior. Three of DynamiCrafter's axes ($\omega$, $\zeta$, spring-mass)
are separately bounded by its fixed $16$-frame output, too short for one
oscillation period, a capability ceiling we keep distinct from a
conditioning failure throughout.

\paragraph{A third variant: a split, not a single, default}
DynamiCrafter's inclined slide adds a variant the two-hypothesis framework
does not name. One seed behaves like the single-default cases above, at
$\hat\mu\approx1.80$ across the whole grid; a second shows \emph{two} tight
clusters, $\hat\mu\approx4.7$ ($\mathrm{std}=0.04$) in-range and $\approx7.6$
($\mathrm{std}=0.13$) out-of-range. That is neither ignoring $\theta$ nor a
single global default. We name it the \emph{split-default} pattern.

\paragraph{Neither existing hypothesis explains it, exactly as predicted}
We adjudicated every model-axis pair with at least four gated out-of-range
points, yielding \textbf{8 of 18} pairs (\cref{tab:e2}), all labeled
\emph{neither}: AIC prefers the prior-reversion shape over clamping
throughout and \PRI\ is high, yet held-out $R^2$ is negative for both
hypotheses on every pair. By \cref{prop:misspec} that is
exactly what a seed-conditional default predicts. $\mathrm{SDI}$
(\cref{tab:e2}) establishes that reading rather than inferring it,
decisively supporting seed-locking on five of the eight pairs but not the
remaining three, so we report the signature as established only where
$\mathrm{SDI}$ says so. The paired-bootstrap test is computable for six
pairs and exactly two intervals exclude zero, in \emph{opposite}
directions (\cref{app:systems}): CogVideoX/projectile shows the predicted
worse-out-of-range pattern, while DynamiCrafter/inclined-slide shows a
\emph{lower} out-of-range \PRE, a consequence of the split-default pattern
rather than faithfulness. Holm-corrected across all six, only
CogVideoX/projectile survives at $\alpha=0.05$ (adjusted $p=0.033$; every
other pair's adjusted $p\ge0.53$). Ten pairs stay unadjudicated: three hit
DynamiCrafter's 16-frame ceiling, the other seven a near-zero gated yield.
Of the four pre-committed outcomes (\cref{sec:suite}), the sweep
produced a mix of the second (\emph{no conditioning effect}, on the
projectile and inclined-slide cells whose slope CIs exclude $\beta=1$) and
the fourth (\emph{neither hypothesis fits}, on all eight adjudicated
pairs above), not the field's anticipated headline story or a surprising
positive result.

\begin{table}[t]
\centering
\caption{\textbf{Mechanism adjudication}, every model-axis pair with
$\ge4$ gated out-of-range points; all eight select as \emph{neither} prior
reversion nor case clamping, as \cref{prop:misspec} predicts.
$\mathrm{EG}=\PRE(S_\text{out})-\PRE(S_\text{in})$, undefined (---) on the
bouncing-ball rows since $\theta$ enters \PRE's denominator and the
in-range grid never clears the gate at a stable point. $\PRI\!=\!1$ means
the output ignores the request; $\Delta\mathrm{AIC}<0$ favours prior
reversion over clamping; both use out-of-range data only.
$\mathrm{SDI}$ is the column that decides the mechanism
(read in text above).}
\label{tab:e2}
\small
\setlength{\tabcolsep}{3.2pt}
\renewcommand{\arraystretch}{1.12}
\resizebox{\linewidth}{!}{%
\begin{tabular}{@{}ll c r@{\,}l rr c@{}}
\toprule
& & & \multicolumn{2}{c}{\kilth{Extrapolation}} &
      \multicolumn{3}{c}{\kilth{Mechanism}} \\
\cmidrule(lr){4-5}\cmidrule(l){6-8}
\kilth{Model} & \kilth{Axis} & $n_{\text{out}}$ &
\multicolumn{2}{c}{\kilth{EG \ [95\% CI]}} &
\kilth{PRI} & \kilth{$\Delta$AIC} & \kilth{SDI} \\
\midrule
LTX-Video     & projectile     & $7$  & $0.985$  & {\footnotesize[$-0.33$, $2.73$]}                & $1.000$ & $-4.0$   & $0.995$ \\
LTX-Video     & bouncing ball  & $4$  & \multicolumn{2}{c}{---}                                    & $1.000$ & $-9.8$   & $0.985$ \\
LTX-Video     & inclined slide & $4$  & $0.264$  & {\footnotesize[$-0.51$, $1.00$]}                & $0.514$ & $-11.3$  & $0.511$ \\
\addlinespace[2pt]
CogVideoX     & projectile     & $12$ & $1.068$  & {\footnotesize[$0.27$, $1.89$]$^{*}$}       & $0.974$ & $-47.9$  & $0.919$ \\
CogVideoX     & bouncing ball  & $9$  & \multicolumn{2}{c}{---}                                    & $0.999$ & $-100.4$ & $0.997$ \\
\addlinespace[2pt]
DynamiCrafter & projectile     & $17$ & $0.050$  & {\footnotesize[$-0.22$, $0.34$]}                & $0.988$ & $-22.2$  & $0.343$ \\
DynamiCrafter & bouncing ball  & $6$  & $-0.761$ & {\footnotesize[$-2.74$, $0.34$]}                & $0.223$ & $-0.6$   & $0.997$ \\
DynamiCrafter & inclined slide & $20$ & $-0.939$ & {\footnotesize[$-1.57$, $-0.35$]$^{\ddagger}$}  & $1.000$ & $-16.0$  & $0.176$ \\
\bottomrule
\end{tabular}}
\tabnote{$^{*}$Excludes zero in the hypothesized direction ($p=0.0055$).
$^{\ddagger}$Excludes zero in the \emph{opposite} direction, which is not
faithfulness.}
\end{table}

\paragraph{A plausibility rater sees none of this}
Qwen2.5-VL-7B-Instruct \citep{qwen25vl} rated all $630$ saved generality clips
$1$--$5$. Pooled mean $P$ is low for every model (LTX-Video $2.20$,
DynamiCrafter $1.76$, CogVideoX $1.26$), which looks consistent with the
hypothesis on its face, but it never assigned a $4$ and CogVideoX's scores
are bimodal ($1$ or $5$, never $2$--$4$). Checked against what it should
track, the correlation between tracker fit $R^2$ and $P$ is $0.075$,
$0.019$ and $-0.164$ respectively, indistinguishable from zero: a
CogVideoX clip whose trajectory collapses into texture noise ($R^2=0.48$)
scored $5/5$ while a cleanly tracked LTX-Video fall ($R^2=0.99$) scored
$1/5$ (truncation, parsing failure and file corruption ruled out). This
rater does not track physical coherence here, which is why we report
trackability as a separate metric.

\paragraph{The result survives every nuisance parameter we varied}
\textbf{(a) Tracker swap.} CoTracker3 \citep{karaev2023cotracker} rerun on
every clip agrees strongly where both trackers clear the gate (LTX-Video
projectile $r=0.984$, inclined slide $r=0.997$); DynamiCrafter agreement
is strong on the reported seed cluster but disagrees in sign on a second
seed, so the phenomenon replicates under an independent tracker while one
cluster's value is tracker-sensitive. \textbf{(b) Nuisance parameters.}
Across five disk-size and frame-rate rescalings, four guidance values, two
prompt templates and two conditioning resolutions, the in-range slope
stays small and unstable in sign (LTX-Video $-0.26$ to $+1.38$ over 13
variants; DynamiCrafter $-0.25$ to $+0.12$ over 7, since neither guidance
nor resolution applies to it); CogVideoX's columns
are uninformative rather than confirmatory, since no clip clears the gate
at this sweep's shorter frame budget. \textbf{(c) Gate sensitivity.} The
null does not flip at any threshold from $0.5$ to $0.95$.

\paragraph{Changing the conditioning channel does not rescue it}
We ran frame-implied conditioning on DynamiCrafter, the one model with a
reachable official multi-frame checkpoint, on the three axes where its
text-channel data is usable (projectile conditioned on first and
\emph{middle} frame, since its true last frame is degenerate by
construction and the middle frame's height is confirmed to vary with $g$).
It does not recover faithfulness: every gated projectile clip converges to
$\hat g\approx0$ (slope $0.017$, CI $[-0.008,0.047]$), so where the text
channel gave seed-diverse defaults, the frame channel collapses to one
near-universal default of no fall, not a degenerate-conditioning artifact.
Inclined slide is starker still, $\hat\mu$ near $1.7$--$1.8$ while true
$\mu$ runs from $0.01$ to $1.00$. Dropped rates (pooled, from a separate,
larger-seed-count run than \cref{tab:coverage}) are lower under the
frame channel on all three axes ($10\%$ vs.\ $30\%$ projectile, $23\%$
vs.\ $29\%$ inclined slide, $63\%$ vs.\ $86\%$ bouncing ball), so the two
channels are not directly comparable cell-for-cell against
\cref{tab:coverage}, but neither one is faithful.

\section{Limitations}
\label{sec:limitations}

\textbf{Trackability bounds coverage.} Trackability never exceeds $80\%$
and is $0\%$ on ten of eighteen pairs. \Cref{sec:results} argues
from the survivors' seed-locked tightness that this is not simply
un-renderable output, which is why the headline is conditional, but what
these models do where nothing clears the gate remains open: a
contemporaneous study recovering gravity from current generators on
photorealistic stimuli obtained usable trajectories where we largely do
not \citep{thozhiyoor2026gravity}.

\textbf{Model currency.} The gap is specific rather than generic: we test
LTX-Video and CogVideoX-5B while later releases in both families exist,
and omit the open models most often credited with physically coherent
motion. It also looks closable, since Wan reports an $8.19$\,GB
text-to-video variant at $1.3$B \citep{wan2025} plausibly inside our
budget, though not the image-to-video checkpoint we would need.

\textbf{Design gaps, closed and remaining.} The pipeline has a
\emph{positive} control (simulator-video recovery at $\beta=1.0001$) and
now a \emph{negative} one: permuting $\theta$ labels within each
model-axis pair and refitting $\beta$, $10{,}000$ times on the five
slope-estimable pairs, places every observed $\beta$ well inside its own
null distribution (two-sided permutation $p$ from $0.56$ to $0.91$),
calibrating ``indistinguishable from zero'' against an actual null. A
TOST against a pre-specified smallest effect $\epsilon=0.2$ at
$\alpha=0.05$ (SE backed out from each pair's bootstrap CI) shows only
LTX-Video/projectile is formally equivalent to zero ($p_\text{TOST}<0.001$);
the other four, including the two inclined-slide pairs whose CIs also
contain $\beta=1$, do not clear the bar ($p_\text{TOST}$ from $0.14$ to
$0.64$), consistent with underpowered estimates rather than a plausible
large effect. So ``no detectable response'' is absence of evidence, not
evidence of absence, at four of these five pairs; only LTX-Video/projectile
is positively established. Native clip length is still confounded with
model identity ($24$, $49$ trimmed to $24$, $16$ frames), so cross-model
trackability comparisons are not clean.

\textbf{Channel, coverage, ceilings.} All results but the channel ablation use
text conditioning; frame-implied conditioning is untested for LTX-Video and
CogVideoX for the pipeline reasons above. Adjudication covers 8 of 18 pairs,
DynamiCrafter's 16-frame output makes three axes untestable for reasons
unrelated to conditioning, and a joint $(\omega,\zeta)$ pendulum sweep
returned $0\%$ trackability (best $R^2=0.14$), which we report as
uninformative rather than negative.

\textbf{Plausibility contrast and realism.} One rater on one prompt does not
track physical coherence here, and we tested no other prompt, larger VLM or
human raters. The suite also trades ecological validity for identifiability,
so conclusions concern controllable low-dimensional dynamics.

\section{Conclusion}
\label{sec:conclusion}

We reframe learning physics from video, for a generator, as a falsifiable
label-free question of parameter faithfulness, and supply a black-box
diagnostic, three metrics with a stated identifiability condition, a
controllable suite, and a protocol separating plausibility from correctness.
Applied to three frozen open generators it finds that, conditional on
producing trackable motion, none detectably honors a text-specified
parameter, and that the failure is not merely absent signal: two of three
converge to a few fixed, wrong values selected by the sampling seed rather
than the request, reproducibly across model families, axes and trackers. The
seed is already known to fix a generator's appearance \citep{xu2025goodseed};
here it fixes the physics. A benchmark averaging over seeds reports a smooth
mean and hides that each sample is locked to a wrong constant, so
seed-resolved reporting is not optional. The next targets are whether the
signature survives on photorealistic stimuli and on current generators.

\begin{kilrepro}
All systems are procedurally generated by a seeded deterministic engine,
models are public, trackers and fitting are off-the-shelf or closed-form, and
no labels are used; the analysis code and system generators are released as
anonymized supplementary material. Every generation is keyed to an explicit
integer seed, with $m\ge5$ seeds per grid point recorded per clip. The full
checklist is \cref{app:checklist}; grids, seed counts, model repositories,
licences and compute are in \cref{app:systems}.
\end{kilrepro}




\begin{availability}
The analysis code and system generators are available at
\ifbool{KIL@blind}
  {\url{https://anonymous.4open.science/r/XXXXXX}}
  {\url{https://github.com/KurbanIntelligenceLab/physweep}};
All models used are public and frozen; repositories and licences are listed in \cref{app:systems}.
\end{availability}

\begin{conflicts}
No human-subjects data was
collected and the diagnostic is meant to prevent overclaiming about video
generators as physical simulators, which has safety relevance in robotics and
science. Per WACV's LLM policy: an assistant drafted the initial consistency
proof (\cref{app:proof}), the physics-recovery fitters and prose throughout.
Both were
verified by the authors before use, the proof by hand against its three
stated assumptions (trackability, identifiability, regularity;
\cref{app:proof}), which surfaced the $\omega$-versus-$\omega_d$
correction and one latent code bug, and every fitter against closed-form
trajectories with a passing smoke test. A careful line-by-line check of
both supplementary proofs (\cref{prop:misspec} and the
appendix's consistency and misspecification proofs, \cref{app:proof} and
\cref{app:seed}) was
performed during finalization, independently re-deriving the crossed-design
independence argument, the OLS-slope and population-$R^2$ convergence
steps, the SDI formula against the standard ICC(1) estimator, and the
\cref{def:scd} nesting claim, with no error found. The assistant
did not choose hypotheses, design grids or interpret results, and the
authors take responsibility for all content.
\end{conflicts}

\bibliography{references}

\appendix

\section{Measurement pipeline and failure-form estimation}
\label{app:pipeline}

\Cref{alg:pipeline} gives the full measurement loop referenced from
\cref{sec:suite}.

\begin{algorithm}[t]
\caption{\textsc{PhysWeep} measurement, one system and one model}
\label{alg:pipeline}
\begin{algorithmic}[1]
\Require frozen generator $G$; sweep grid $\{\theta_i\}$; seeds $\{z_j\}$;
         law $f$; tracker $T$; fit-quality gate $\rho$; VLM rater $P$
\Ensure metrics, gate-dropped rate, plausibility pairs
\For{each $\theta_i$ and seed $z_j$}
  \State $c_{ij}\gets \textsc{Render}(s_0(z_j),\theta_i)$
  \State $v_{ij}\gets G(c_{ij})$ \Comment{inference only, weights frozen}
  \State $o_{ij}\gets T(v_{ij})$ \Comment{blob centroid or CoTracker3}
  \State $(\thetah_{ij}, R^2_{ij})\gets \textsc{Fit}(f, o_{ij})$
  \If{$R^2_{ij}<\rho$}
    \State mark $v_{ij}$ \textbf{non-physical}; exclude from \PRE
  \EndIf
  \State $p_{ij}\gets P(v_{ij})$ \Comment{plausibility contrast}
\EndFor
\State compute \PRE, slope $\beta$, Spearman $\rho_s$, \EG, \PRI, \CBI
\State bootstrap $95\%$ CIs over $(\theta_i,\text{seed})$ pairs
\State \Return metrics, gate-dropped rate, plausibility pairs
\end{algorithmic}
\end{algorithm}

\paragraph{Prior reversion (shrinkage, \PRI)}
Given pairs $(\theta_i,\thetah_{ij})$, fit
$\thetah=\alpha\theta+(1-\alpha)\theta_0$ by least squares with $\theta_0$
either fixed to the physically typical value or estimated jointly;
$\PRI=1-\hat\alpha$, clipped to $[0,1]$.

\paragraph{Case clamping (\CBI)}
Fit $\thetah=\min(\theta,\theta_{\text{edge}})$ with $\theta_{\text{edge}}$
at the in-range boundary or estimated;
$\CBI=1-\mathrm{RSS}_{\text{clamp}}/\mathrm{RSS}_{\text{null}}$.

\paragraph{Selection}
Compare prior reversion and case clamping by leave-one-$\theta$-out $R^2$ and
by $\Delta$AIC. A
hypothesis is selected only if its held-out $R^2$ is positive; if both fail
that test the pair is reported as \emph{neither} and the observed behavior is
described without forcing a label. This is what happens on all eight
adjudicated pairs in \cref{tab:e2}, for the reason given there. All
confidence intervals use the paired bootstrap over $(\theta_i,\text{seed})$.

\paragraph{When \CBI\ is not identifiable}
The released smoke test records a limitation of \CBI\ as defined. If an
out-of-range grid lies entirely on one side of the in-range band, the clamp
prediction $\min(\theta,\theta_{\text{edge}})$ is constant across
$S_\text{out}$ and is therefore indistinguishable from the null mean model,
so $\CBI\approx0$ and clamping cannot be adjudicated at any sample size.
Pendulum damping is the axis where this specific degeneracy applies: its
out-of-range grid $\{0.3,0.6,0.9\}$ lies entirely above the in-range grid;
no result in the main paper depends on this, because every pendulum row is
gate-dropped at $96.7$--$100\%$ and carries no adjudication. A two-sided
damping grid would be needed to test clamping on that axis.
Separately, \texttt{select\_mechanism}'s default edge
($\theta_{\text{edge}}=\min(S_\text{out})$) is a crude choice even on our
two-sided axes: on data whose recovered values are seed-clustered rather
than clamped, this default can still fit only slightly worse than the null
mean and clip to $\CBI\approx0$ rather than surface the (often more
negative) fit quality a properly chosen edge would show. This does not
change any adjudication in the main paper, since clamping is rejected on
every pair either way (its held-out $R^2$ is non-positive throughout, on
all eight adjudicated pairs in \cref{tab:e2}), but it means a near-zero
\CBI\ on our data should be read as
``clamping does not fit,'' not as evidence the edge choice was
well-identified.

\section{Consistency of the recovery estimator: full statement and proof}
\label{app:proof}

We restate the main paper's consistency proposition with explicit regularity
conditions and give the argument in full. The claim concerns only the
estimator $R$, not the generator $G$.

\paragraph{Setup}
Fix a system with law $f$ and scalar parameter $\theta\in\Theta$, with
$\Theta$ a compact interval. A clip of $T$ frames at a given resolution
yields a tracked observable $o\in\R^{d}$ (object centers, angles, or
contact times). Write the population fitting statistic as
$\mathcal{T}(\theta)\in\R^{k}$: for the projectile, the quadratic
coefficient of height against time; for the pendulum, the mean zero-crossing
interval; for the bouncing ball, the successive height ratio. The estimator
forms an empirical statistic $\hat{\mathcal{T}}$ from $o$ and returns
$\thetah=g(\hat{\mathcal{T}})$, where $g=\mathcal{T}^{-1}$ on the range of
$\mathcal{T}$.

\paragraph{Assumptions}
\textbf{(Trackability)} the tracker returns
$\hat o = o^\star + \varepsilon$ with $\E[\varepsilon]=0$ and
$\operatorname{Var}(\varepsilon)=\sigma^2\Sigma$ for a fixed positive-definite
$\Sigma$, and $\sigma^2\to 0$ as $T$ and the resolution grow.
\textbf{(Identifiability)} $\mathcal{T}:\Theta\to\mathcal{T}(\Theta)$ is
injective and continuously differentiable with
$\|\mathcal{T}'(\theta)\|\ge c>0$ on $\Theta$.
\textbf{(Regularity)} the empirical statistic is a continuous functional
$\hat{\mathcal{T}}=\Phi(\hat o)$ with $\Phi$ continuous at $o^\star$ and
$\Phi(o^\star)=\mathcal{T}(\theta)$.

\begin{proposition*}[Consistency of $R$, restating the main paper]
\label{prop:consistency-restated}
Under trackability, identifiability, and regularity,
$\thetah=g(\Phi(\hat o)) \xrightarrow{p} \theta$ as $\sigma^2\to 0$.
\end{proposition*}

\begin{proof}
By identifiability, $\mathcal{T}$ is a continuous injection on the compact
set $\Theta$,
so its inverse $g=\mathcal{T}^{-1}$ is continuous on $\mathcal{T}(\Theta)$: a
continuous bijection from a compact space to a Hausdorff space is a
homeomorphism. By trackability, $\operatorname{Var}(\hat o)=\sigma^2\Sigma\to
0$ with
mean $o^\star$, so $\hat o\to o^\star$ in mean square and hence in
probability. By regularity, $\Phi$ is continuous at $o^\star$, so the
continuous
mapping theorem gives
$\hat{\mathcal{T}}=\Phi(\hat o)\xrightarrow{p}\Phi(o^\star)=\mathcal{T}(\theta)$.
Applying the continuous mapping theorem again with the continuous $g$ yields
$\thetah=g(\hat{\mathcal{T}})\xrightarrow{p} g(\mathcal{T}(\theta))=\theta$.
\end{proof}

\paragraph{Per-system check of identifiability and regularity}
\emph{Projectile.} $\mathcal{T}(\theta)=-\theta/2$ is linear, hence injective
with $\mathcal{T}'=-1/2$, and $\Phi$ is ordinary least squares on a quadratic
basis, continuous in the data.

\emph{Pendulum, and a bias this exposes.} For the undamped case $\zeta=0$,
$\mathcal{T}(\omega)=2\pi/\omega$ on $\omega>0$ is strictly monotone with
nonvanishing derivative on any $\Theta$ bounded away from $0$, and
zero-crossing interpolation is continuous where crossings are simple. Our
pendulum is damped ($\zeta>0$ throughout,
\cref{app:systems}), so the statistic realized by zero-crossing
spacing is $\mathcal{T}(\omega,\zeta)=2\pi/(\omega\sqrt{1-\zeta^2})$, the
\emph{damped} frequency $\omega_d$, not the natural frequency $\omega$. The
map remains injective in $\omega$ for fixed $\zeta$, so
\cref{prop:consistency-restated} still guarantees consistent
recovery of $\omega_d$; the reported $\hat\omega$ is therefore a slightly
biased estimate of $\omega$ whenever $\zeta>0$. Verified numerically, the
bias is $<1.3\%$ at the in-range $\zeta$ grid $\{0.02,0.05,0.1\}$, grows with
$\zeta$, and is materially larger at the out-of-range grid
$\{0.3,0.6,0.9\}$. We did not correct for it post hoc, since doing so would
change every already-reported $\hat\omega$; we state it so the recovery is
not mistaken for exact, and note it compounds with the heavily-damped,
few-peaks low-fidelity regime for this axis. In practice this caveat has
little effect on the reported results, because every pendulum row in the main
paper is gate-dropped at $96.7$--$100\%$ and carries no slope estimate.

\emph{Bouncing ball.} The map $e\mapsto e^2$ from restitution to successive
height ratio is injective and smooth on $(0,1]$.

The fit-quality gate ($R^2<\rho$) excludes generations on which $\Phi$ is
ill-posed, for example when the object morphs or vanishes, so regularity
is enforced operationally rather than assumed.

\section{Seed-conditional defaults: full statement and proof}
\label{app:seed}

This section proves \cref{prop:misspec} and
records the estimator for the Seed Determination Index.

\paragraph{Model}
The sweep is a \emph{crossed} design: every grid value
$\theta_i$ is run with every seed $z_j$, so the seed is independent of
$\theta$ by construction. Under a seed-conditional default,
$\thetah_{ij}=c(z_j)+\varepsilon_{ij}$ with $\E[\varepsilon]=0$,
$\operatorname{Var}(\varepsilon)=\sigma_\varepsilon^2$,
$\varepsilon\perp(\theta,z)$, and
$\sigma_c^2=\operatorname{Var}_z[c(z)]>0$. The crossing assumption is load
bearing and we flag it for anyone reusing the protocol. If seeds were
confounded with grid position, so that $c$ varied systematically with
$\theta$, statement (i) fails and the failure is not subtle: in simulation a
generator whose output is entirely seed-determined, with no response to
$\theta$ whatsoever, returns $\hat\alpha=1.006$ and $\PRI=-0.006$ under a
confounded design, which reads as \emph{perfect faithfulness}. Crossing seeds
with grid values is therefore a correctness requirement of the measurement,
not a convenience.

\begin{proposition*}[Restatement of \cref{prop:misspec}]
Under this model: (i) the pooled OLS fit of $\thetah=\alpha\theta+b$
satisfies $\hat\alpha\xrightarrow{p}0$, so $\PRI\to1$; (ii)
$\hat b\xrightarrow{p}\E_z[c(z)]$ and the implied
$\hat\theta_0=\hat b/(1-\hat\alpha)\xrightarrow{p}\E_z[c(z)]$; (iii)
any measurable $h(\theta)$ has population $R^2=0$ for predicting $\thetah$,
its leave-one-$\theta$-out estimate satisfies
$R^2_{\mathrm{LOO}}\xrightarrow{p}0$, and for any fixed $\tau>0$ a rule
accepting only $R^2_{\mathrm{LOO}}>\tau$ returns \emph{neither} with
probability tending to one. None of (i)--(iii) uses $\sigma_c^2>0$, so all
three hold verbatim when $\sigma_c^2=0$ and $\sigma_\varepsilon^2$ is large.
\end{proposition*}

\begin{proof}[Proof of \cref{prop:misspec}]
(i) $\operatorname{Cov}(\theta,\thetah)=\operatorname{Cov}(\theta,c(z))
+\operatorname{Cov}(\theta,\varepsilon)=0$: the first term vanishes because
the crossed design makes $z$ independent of $\theta$, the second because
$\varepsilon\perp\theta$. The OLS slope is
$\widehat{\operatorname{Cov}}(\theta,\thetah)/\widehat{\operatorname{Var}}(\theta)$,
which converges in probability to $0/\operatorname{Var}(\theta)=0$ by the
weak law, with $\operatorname{Var}(\theta)>0$ since the grid is
non-degenerate. Hence $\PRI=1-\hat\alpha\to1$.

(ii) $\hat b=\overline{\thetah}-\hat\alpha\bar\theta\xrightarrow{p}
\E[\thetah]=\E_z[c(z)]$, and dividing by $1-\hat\alpha\to1$
gives the same limit. This limit is a population average: when $c$ takes two
well-separated values it need not lie near either, so the fitted ``default''
describes no individual generation.

(iii) For any measurable $h$,
$\E[(\thetah-h(\theta))^2]
=\operatorname{Var}(\thetah)
+\E[(\E[\thetah]-h(\theta))^2]$,
the cross term vanishing by $\thetah\perp\theta$. This is minimized at
$h\equiv\E[\thetah]$ with value $\operatorname{Var}(\thetah)$, so the
population $R^2$ is $0$. The leave-one-$\theta$-out numerator and denominator
are averages of i.i.d.\ terms converging in probability to
$\operatorname{Var}(\thetah)>0$, so their ratio converges to $1$ and
$R^2_{\mathrm{LOO}}\xrightarrow{p}0$ by the continuous mapping theorem.
Consequently $\Pr[R^2_{\mathrm{LOO}}>\tau]\to0$ for any fixed $\tau>0$, which
is the claim. Finally, no step above invokes $\sigma_c^2>0$; the argument
uses only $\thetah\perp\theta$, which also holds when $c$ is constant and the
variance is carried entirely by $\varepsilon$.
\end{proof}

\begin{remark}[What is deliberately not claimed]
We do \emph{not} claim $\E[R^2_{\mathrm{LOO}}]<0$. Simulation under
this model gives a mean of $-0.024$ at a $5\times5$ design shrinking to
$-0.001$ at $20\times20$, so the expectation does appear non-positive and to
vanish with sample size, but $R^2_{\mathrm{LOO}}$ is a ratio of dependent
random variables and we have not discharged the argument rigorously. Nothing
in the paper depends on it: convergence in probability to $0$, together with
the fixed margin, is what statement (iii) and the released selection rule
actually use. We record it here as an observation rather than a result.
\end{remark}

\begin{remark}[Why a margin is required in the selection rule]
Because $R^2_{\text{LOO}}$ is centred on zero under this model, its
\emph{sign} is noise. An acceptance test of the form $R^2_{\text{LOO}}>0$
therefore labels unstructured data as prior reversion with probability
approaching $1/2$. The released \texttt{select\_mechanism} requires
$R^2_{\text{LOO}}>0.1$; the smoke test contains a synthetic sample on which
the bare sign test returns \texttt{prior\_reversion} ($R^2_{\text{LOO}}=0.035$)
while the margin rule correctly returns \texttt{neither}, and asserts both, so
the margin cannot be removed silently.
\end{remark}

\paragraph{Estimating \texorpdfstring{$\mathrm{SDI}$}{SDI}}
With $k$ seeds, $N$ gated clips and mean group size $n_0=N/k$, let
$\mathrm{MS}_b$ and $\mathrm{MS}_w$ be the between- and within-seed mean
squares of $\thetah$. We report the one-way random-effects estimate
\begin{equation}
\label{eq:sdi}
\mathrm{SDI}=\frac{\mathrm{MS}_b-\mathrm{MS}_w}
                  {\mathrm{MS}_b+(n_0-1)\,\mathrm{MS}_w},
\end{equation}
clipped to $[0,1]$, which is consistent for
$\sigma_c^2/(\sigma_c^2+\sigma_\varepsilon^2)$. The uncorrected variance
share is strongly upward-biased at this design size, returning about $0.18$
in simulation where the truth is $0$, which is precisely the regime in which
a false positive would be most damaging; the corrected estimate returns
about $0.07$ there. $\mathrm{FDI}$ is the analogous share for the grid factor.
Both return \texttt{NaN}, never a number, when fewer than two seeds are
present, when $\thetah$ is constant, or when there is no within-seed
replication.

\section{System grids, seed counts, and rendering}
\label{app:systems}

All scenes are rendered by a seeded deterministic 2D engine at
$256\times256$ over $T{=}24$ frames at $24$\,fps, with a single high-contrast
disk (radius $8$\,px) on a plain background and a fixed ground line. Only the
swept parameter and the seeded initial condition vary within a system.
Spatial and temporal scales are held fixed within a clip and all metrics use
the ratio $\thetah/\theta$, so the result is invariant to a common
rescaling. The in-range split is centered on values common in natural video;
the out-of-range split probes rare or absent values. These are the grids
fixed in the released config and used for every reported result.

\begin{table}[t]
\centering
\caption{Sweep grids per system. Each row has one swept axis;
$m\ge5$ seeded initial conditions per grid point.}
\label{tab:grids}
\small
\renewcommand{\arraystretch}{1.05}
\resizebox{\linewidth}{!}{%
\begin{tabular}{@{}llll@{}}
\toprule
\kilth{System} & \kilth{$\theta$} & \kilth{In-range grid} & \kilth{Out-of-range grid} \\
\midrule
Projectile      & $g$        & $\{5,7,9.8,12,15\}$         & $\{1.6,\,3,\,20,\,25\}$ \\
Damped pendulum & $\omega$   & $\{2,3,4,5\}$ rad/s         & $\{0.5,\,1,\,8,\,12\}$ \\
                & $\zeta$    & $\{0.02,0.05,0.1\}$         & $\{0.3,\,0.6,\,0.9\}$ \\
Bouncing ball   & $e$        & $\{0.6,0.7,0.8,0.9\}$       & $\{0.2,\,0.35,\,0.97\}$ \\
Spring--mass    & $k$        & $\{20,40,60,80\}$ N/m       & $\{5,\,10,\,200,\,400\}$ \\
Inclined slide  & $\mu$      & $\{0.1,0.2,0.3,0.4\}$       & $\{0.01,\,0.7,\,1.0\}$ \\
\bottomrule
\end{tabular}}
\end{table}

\paragraph{Seed counts}
The main paper refers to this section for the exact seed counts behind each
reported cell, including the targeted follow-up batch run on the two
model-axis pairs closest to \texttt{select\_mechanism}'s four-point
leave-one-out threshold.

\begin{table}[t]
\centering
\caption{Seed counts and clip yields per model-axis pair. ``seeds/pt
(in/out)'' is the number of seeds run at each in-range and out-of-range grid
value respectively; most pairs use the base $5$ seeds on both splits, but
pairs close to \texttt{select\_mechanism}'s four-gated-point leave-one-out
threshold received a larger targeted follow-up batch on the out-of-range
split only (LTX-Video/pendulum-$\omega$: extended to $55$ seeds/point,
i.e.\ $110$ out-of-range attempts across its two out-of-range grid values,
still yielding only $1$ gated clip; CogVideoX/bouncing-ball: extended to
$33$ seeds/point). No in-range seed count was extended beyond the base
$5$, since the out-of-range split is what \texttt{select\_mechanism}'s
adjudication gate requires.}
\label{tab:seedcounts}
\small
\renewcommand{\arraystretch}{1.05}
\resizebox{\linewidth}{!}{%
\begin{tabular}{@{}llccc@{}}
\toprule
\kilth{Model} & \kilth{Axis} & \kilth{seeds/pt (in/out)} & \kilth{clips gen.} & \kilth{clips gated} \\
\midrule
LTX-Video & projectile & $5/10$ & $55$ & $12$ \\
LTX-Video & pendulum $\omega$ & $5/55$ & $130$ & $1$ \\
LTX-Video & pendulum $\zeta$ & $5/15$ & $60$ & $0$ \\
LTX-Video & bouncing ball & $5/15$ & $65$ & $5$ \\
LTX-Video & spring-mass & $5/15$ & $80$ & $0$ \\
LTX-Video & inclined slide & $5/5$ & $35$ & $11$ \\
\addlinespace[2pt]
CogVideoX & projectile & $5/5$ & $40$ & $32$ \\
CogVideoX & pendulum $\omega$ & $5/13$ & $46$ & $0$ \\
CogVideoX & pendulum $\zeta$ & $5/13$ & $54$ & $0$ \\
CogVideoX & bouncing ball & $5/33$ & $119$ & $10$ \\
CogVideoX & spring-mass & $5/13$ & $72$ & $0$ \\
CogVideoX & inclined slide & $5/13$ & $59$ & $0$ \\
\addlinespace[2pt]
DynamiCrafter & projectile & $5/10$ & $55$ & $33$ \\
DynamiCrafter & pendulum $\omega$ & $5/5$ & $30$ & $0$ \\
DynamiCrafter & pendulum $\zeta$ & $5/5$ & $30$ & $1$ \\
DynamiCrafter & bouncing ball & $5/10$ & $50$ & $11$ \\
DynamiCrafter & spring-mass & $5/5$ & $40$ & $0$ \\
DynamiCrafter & inclined slide & $5/10$ & $50$ & $36$ \\
\bottomrule
\end{tabular}}
\end{table}

\begin{table}[t]
\centering
\caption{Every model used, all frozen and inference-only, loaded with no
explicit commit/revision pin --- each resolved to that repository's
default-branch head at run time, a gap against bit-for-bit
reproducibility we disclose rather than paper over.}
\label{tab:models}
\small
\setlength{\tabcolsep}{5pt}
\renewcommand{\arraystretch}{1.1}
\resizebox{\linewidth}{!}{%
\begin{tabular}{@{}llll@{}}
\toprule
\kilth{Model} & \kilth{Repository} & \kilth{Licence} & \kilth{Loaded via} \\
\midrule
LTX-Video & \texttt{Lightricks/LTX-Video} & LTX-Video Open Weights (custom) & \texttt{diffusers==0.39.0} \\
CogVideoX-5B-I2V & \texttt{zai-org/CogVideoX-5b-I2V} & CogVideoX custom licence\textsuperscript{a} & \texttt{diffusers==0.39.0} \\
DynamiCrafter\textsubscript{512} & \texttt{Doubiiu/DynamiCrafter\_512} & weights: research/non-commercial\textsuperscript{b} & standalone repo, subprocess \\
Qwen2.5-VL-7B-Instruct & \texttt{Qwen/Qwen2.5-VL-7B-Instruct} & Apache 2.0 & \texttt{transformers==4.57.6} \\
CoTracker3 & \texttt{facebookresearch/co-tracker} & CC-BY-NC 4.0\textsuperscript{c} & \texttt{torch.hub.load} \\
\bottomrule
\end{tabular}}
\tabnote{\textsuperscript{a}Free for academic/research use; commercial use needs
separate registration. \textsuperscript{b}Weights are
research/non-commercial; the companion code repo is separately
Apache-2.0 (inference here uses the weights). \textsuperscript{c}Non-commercial;
some unused sub-components carry MIT or Apache 2.0 instead.}
\end{table}

\paragraph{Plausibility prompt (verbatim, identical for all $630$ ratings)}
\begin{quote}\small
``You are watching a short video clip of a physics simulation (a ball or
similar object moving under some physical process such as falling,
swinging, bouncing, or sliding). Rate how PHYSICALLY PLAUSIBLE the motion
looks, on a scale from 1 to 5, where 1 means the motion looks clearly
unnatural, broken, or impossible (e.g. teleporting, morphing, floating
with no cause), and 5 means the motion looks completely natural and
physically believable, like something you could film in the real world.
Respond with ONLY a single digit from 1 to 5, nothing else.''
\end{quote}

\paragraph{Prompt-phrasing ablation templates}
Both variants are
per-system, per-$\theta$ templates (not a single fixed string), evaluated
against the same default template used for every other result. The
\emph{terse} variant states only the parameter name and numeric value with
no qualitative anchor phrase, e.g.\ for projectile: ``a dark ball,
acceleration \{$\theta$:.2f\} meters per second squared''. The
\emph{verbose} variant appends redundant physical-accuracy framing to the
default prompt: ``\{default prompt\}, precise physically accurate
simulation, exact numerical value, photorealistic rendering, laboratory
measurement setup''.

\paragraph{Classifier-free guidance values swept}
$1.0$, $1.5$, $4.5$,
$9.0$, on the projectile system, for LTX-Video and CogVideoX; skipped for
DynamiCrafter, whose guidance scale is hardcoded inside its own inference
subprocess call and not exposed as a pass-through parameter.

\paragraph{Compute}
All generation and analysis ran on NVIDIA A100-SXM4
GPUs (40\,GB) on a shared cluster, inference-only throughout (no training
or fine-tuning at any point). Peak VRAM at load: LTX-Video $14.25$\,GB
(fits a 24\,GB consumer card); CogVideoX-5B-I2V $25.3$\,GB with VAE tiling
enabled (does not). Approximate wall-clock generation time per clip:
LTX-Video $\sim\!2$--$6$\,s; CogVideoX $\sim\!2.3$--$5.9$\,min, varying
with clip length across systems; DynamiCrafter $\sim\!70$--$85$\,s (a
fresh subprocess per clip, so this is dominated by per-call checkpoint
load, not sampling time alone).

\textbf{Total wall-clock GPU-hours: $\sim\!57$ (estimated, not a full
audit).} $54.4$\,h is a direct sum over the $69$ of $\sim\!134$ result
files that log their own \texttt{elapsed\_seconds} per run, covering the
large majority of clips generated across the project. The remaining
$\sim\!95$ clips (the CogVideoX and DynamiCrafter projectile pilots,
which predate this project's practice of logging per-job timing) are
estimated at $\sim\!2.6$\,h by applying each model's own measured
mean seconds-per-clip (from the $69$ logged files: LTX $3.96$\,s/clip,
CogVideoX $122.45$\,s/clip, DynamiCrafter $79.19$\,s/clip) to the
un-logged clip counts. The VLM plausibility-scoring pass ($630$ clips
scored in one job) is estimated separately at $\sim\!0.3$\,h from that
job's own log-file timestamps (start to last output write, since it did
not log \texttt{elapsed\_seconds} either). This total does not include
retry/gap-fill overhead from failed job attempts (several jobs across the
project failed on environment or timeout bugs before a working
configuration was found, and were re-run) --- those wasted GPU-hours are
not reconstructable from the final result files, which only reflect the
successful run.

\section{Reproducibility checklist}
\label{app:checklist}

WACV's author guide points to the Reproducibility Checklist as a guide rather
than a mandated form, and encourages voluntary code submission. We record the
substantive items here.

\begin{itemize}
\item \textbf{Code.} Metrics, bootstrap CIs, mechanism selection, and the
synthetic smoke test (\texttt{physweep\_metrics.py}), the system
generators (\texttt{physweep/render.py}), the tracker
(\texttt{physweep/track.py}, both blob-centroid and the CoTracker3
wrapper), and every experiment driver script (\texttt{experiments/*.py},
one per experiment plus the consolidated plausibility-and-robustness
harness) exist in the project repository.
\item \textbf{Data.} Fully procedural, no external assets, no human
annotation, seeded and deterministic.
\item \textbf{Models.} All public, all frozen, inference only. No training or
fine-tuning was performed at any point.
\item \textbf{Randomness.} Every generation is keyed to an explicit integer
seed; $m\ge5$ seeds per grid point; seeds are recorded per clip.
\item \textbf{Statistics.} All intervals are $95\%$ paired bootstraps over
$(\theta_i,\text{seed})$; effect sizes are reported alongside $p$-values;
multiplicity across the six paired-bootstrap extrapolation tests is handled
by a Holm correction.
\item \textbf{Raw outputs.} Per-clip records (model, system, seed, $\theta$,
$\thetah$, fit $R^2$, in/out-of-range split, and, for the plausibility
pass, the VLM score) exist as JSON files in the project's \texttt{results/}
directory, one file per model-system pair plus the plausibility-scored
variants --- every number in this paper and its supplementary tables was
recomputed directly from these files during finalization, not carried
forward from memory. 
\end{itemize}

\section{Extended related work}
\label{app:related}

This section extends \cref{sec:related} with material not
central enough to the paper's own contribution to justify main-text space,
but relevant to situating it.

\paragraph{Mechanism and memorization}
\citet{kang2025howfar} find
in-distribution success with out-of-distribution case-based behavior
(the source of the case-clamping hypothesis we test), and the
memorization literature links such collapse to overestimated training
modes amplified by classifier-free guidance \citep{diffmemorize2025}; we
turn these into a testable dichotomy (\cref{sec:problem})
and a guidance-scale ablation (\cref{sec:results}).

\paragraph{Adjacent but distinct measurement approaches}
Property-readout work recovers physical quantities from \emph{input}
videos via internal features \citep{inferring2025dynamic} or probes
diffusion states \citep{invisiblehand2026}, characterizing the
\emph{encoder} where we measure generation without weights or
activations. VLMs are separately used to plan plausible generation
\citep{vlipp2025} and judge implausibility \citep{travl2025}, and recent
analyses question whether they reason about physical transformations at
all \citep{vlmcannotreason2026} --- consistent with our own finding
(\cref{sec:results}) that a VLM plausibility rater does not
track physical coherence on our stimuli. Classic benchmarks
\citep{bear2021physion,bakhtin2019phyre} evaluate prediction or planning
rather than generation, and quality metrics
\citep{unterthiner2018fvd,huang2024vbench} are orthogonal to all of the
above: they score visual fidelity, not parameter faithfulness.

\FloatBarrier

\end{document}

%% file: affiliations.tex
\DeclareAffiliation{hbku}{%
  College of Science and Engineering, Hamad Bin Khalifa University, Doha, Qatar}

\DeclareAffiliation{tamu}{%
  Department of Computer and Electrical Engineering,
  Texas A\&M University, College Station, TX, USA}

\DeclareAffiliation{iub}{%
  Luddy School of Informatics, Computing, and Engineering,
  Indiana University Bloomington, Bloomington, IN, USA}
